\documentclass[]{article}

\usepackage{times} % CHeck with Frank
\usepackage{arxiv}

\usepackage{mathtools}  % Loads and improves amsmath
\usepackage{amssymb}    % Loads amsfonts automatically
\usepackage{dsfont}     % For \mathds{R} etc.
\usepackage{amsthm}     % For arxiv

\newtheorem{theorem}{Theorem}

\newtheorem{corollary}{Corollary}
\newtheorem{definition}{Definition}

\usepackage[T1]{fontenc}
\usepackage{xspace}
\usepackage{float}      % Allows [H] placement for figures/tables

\usepackage{graphicx}
\usepackage{adjustbox}
\usepackage{subfig}
\usepackage{booktabs}   % For \toprule, \midrule, \bottomrule
\usepackage{multirow}
\usepackage{caption}

\usepackage{algorithm}   % The "wrapper" (provides the floating environment)
\usepackage{algorithmic} % The "content" (provides the \STATE, \IF, etc.)

\usepackage{xcolor}

\DeclarePairedDelimiter\floor{\lfloor}{\rfloor}
\DeclarePairedDelimiter\ceil{\lceil}{\rceil}
\newcommand{\ignore}[1]{}

\newcommand{\gsemo}[1]{GSEMO\xspace}

\begin{document}

\author{
Liam Wigney\\
Optimisation and Logistics\\
School of Computer Science \\ and Information Technology\\
Adelaide University\\
Adelaide, Australia
\And
Frank Neumann\\
Optimisation and Logistics\\
School of Computer Science \\ and Information Technology\\
Adelaide University\\
Adelaide, Australia
}

\title{Analysis of Multitasking Pareto Optimization for Monotone Submodular Problems
}
\maketitle

\begin{abstract}
	Pareto optimization via evolutionary multi-objective algorithms has been shown to efficiently solve constrained monotone submodular functions. Traditionally when solving multiple problems, the algorithm is run for each problem separately.
	We introduce multitasking formulations of these problems that are an effective way to solve multiple related problems with a single run.
	In our setting the given problems share a monotone submodular function $f$ but have different knapsack constraints.
	We examine the case where elements within a constraint have the same cost and show that our multitasking formulations result in small Pareto fronts. This allows the population
	to share solutions between all problems leading to significant improvements compared to running several classical approaches independently.
	Using rigorous runtime analysis, we analyze the expected time until the introduced multitasking approaches obtain a $(1-1/e)$-approximation for each of the given problems. Our experimental investigations for the maximum coverage problem give further insight into the dynamics behind how the approach works and doesn't work in practice for problems where elements within a constraint also have varied costs.
\end{abstract}

\section{Introduction}
Evolutionary multitasking is a recent evolutionary algorithm approach that allows several (related) problems to be solved using the population of an evolutionary algorithm ~\cite{DBLP:conf/soict/Ong15,Tan2023,DBLP:conf/cec/DaGOF16}. It significantly distinguishes itself from the classical approach of using a population to compute one single near-optimal solution. The approach has found applications in a wide range of areas, including continuous, combinatorial multi-objective optimization, optimization as a service, and learning using genetic programming ~\cite{DBLP:journals/cogcom/OsabaSMH22}.

Evolutionary algorithms are effectively used to solve combinatorial optimization problems, which can often be formulated as submodular monotone problems with a set of constraints ~\cite{DBLP:books/cu/p/0001G14,DBLP:journals/mor/NemhauserW78}. In real-world problems, these are often ones that can be modeled in terms of diminishing returns, but they also encompass theoretical problems such as max coverage ~\cite{DBLP:journals/ipl/KhullerMN99} and max cut ~\cite{DBLP:journals/jacm/GoemansW95}. These functions can be seen as the discrete version of the continuous concept of convexity.

Maximizing these functions with constraints, even ones as simple as a constraint on the number of elements you can choose, is NP-hard. The evolutionary algorithm used in this analysis is the GSEMO (global simple evolutionary multi-objective optimizer), a multi-objective optimizer that can be successfully adapted to single-objective problems. This approach avoids the solution getting trapped in local optimal like EAs such as the $(1+1)$-EA do, although archive-based methods have shown improvements that bring it up to the same performance as GSEMO and other multi-objective approaches ~\cite{DBLP:conf/ppsn/NeumannR24}.

To investigate these problems, we use both rigorous runtime analysis and in depth experimentation with a statistical analysis. In depth explanations for the methodology behind the theoretical aspects can be found in ~\cite{DBLP:series/ncs/2020DN,DBLP:conf/gecco/Witt14a}.

In this paper, we contribute to the fundamental understanding of how the multitasking Pareto approach works when maximizing submodular monotone problems with knapsack constraints. We theoretically analyze the case with unit constraints before experimentally investigating a broader ranger of knapsack constraints.

\subsection{Related Work}

Evolutionary multitasking literature has mostly been focused on continuous problems ~\cite{DBLP:journals/tcyb/LinLTG21,DBLP:journals/tcyb/LinWMGLC24,SUN2023504,DBLP:journals/isci/HuLSM22}, these often contain sophisticated algorithms often utilizing neural networks and deep learning. While effective, analysis of the dynamics behind the improvements are often intractable and so an understanding of how these approaches actually work are lacking. Literature on combinatorial problems is more sparse, with work often focused on variations of the TSP or Knapsack problem ~\cite{DBLP:conf/gecco/WigneyNON25,DBLP:conf/gecco/Don0025,7848632}. There are currently no papers on evolutionary multitasking that focus explicitly on submodular problems.

There are limited theoretical evolutionarily multitasking works, with results on knowledge transfer ~\cite{DBLP:conf/cec/ScottJ24,DBLP:conf/foga/ScottJ23} and for classical benchmarking problems such as OneMax, LeadingOnes, and Jump~\cite{DBLP:conf/gecco/Lengler0N25}. This means the focus of the literature so far has been mostly on the improvements seen rather than on understanding the mechanisms behind multitasking and how it differs compared to the classical methods for optimizing.  

Previous work on runtime analysis for GSEMO and simple single-objective evolutionary algorithms such as $(1+\lambda)$-EA and $(1+1)$-EA have typically been done by considering single unweighted uniform constraints and general cost constraints ~\cite{DBLP:journals/ec/FriedrichN15,DBLP:journals/ai/RoostapourNNF22,DBLP:conf/ppsn/NeumannR24}. Works that include arbitrary weights often consider chance constraints rather than deterministic uniform constraints ~\cite{NeumannNeumannTCS23,DBLP:conf/ppsn/YanNN24}. There is also a growing body of work that considers matroid constraints but only in the static case ~\cite{DBLP:journals/ai/QianYTYZ19,DBLP:conf/ppsn/DoN20}. The more complex constraint types often require sophisticated methods that are not needed in the simple uniform constraint cases; however, they all give reasonable upper bounds on the runtime to reach a $\left(1-\frac{1}{e}\right)$-approximation. More advanced Pareto approaches have also been considered, often removing a factor from the expected time ~\cite{DBLP:conf/ijcai/Crawford21}.

Classically, these constraints are considered to be static throughout the runtime of the algorithm, however literature around dynamic constraints is sparse. Dynamic bounds with approximated costs and approximately submodular functions were analyzed for POMC (Pareto Optimization for maximizing a Monotone function with a monotone Cost constraint)  ~\cite{DBLP:journals/ai/RoostapourNNF22}. Works improving POMC have been considered for dynamic cost constraints ~\cite{DBLP:conf/ijcai/0001Q0021}, with a focus on on more efficient approximation guarantees via the new algorithm FPOMC (Fast Pareto Optimization for maximizing a Monotone function with a monotone Cost constraint).

Excluding both $(1+\lambda)$-EA and $(1+1)$-EA, most of these algorithms are multi-objective Pareto optimization based approaches. They work by building up a Pareto front for the problem, which ensures that optimal solutions remain in the population while suboptimal solutions don't get included.
Due to this, the population size is a significant factor that determines the runtime bounds as noted in previous papers ~\cite{DBLP:conf/ppsn/NeumannR24,DBLP:conf/ppsn/YanNN24}, as it controls the number of points and thus gives the upper bound on the Pareto front. As the Pareto front grows exponentially large as the number of trade-offs increases, for a multitasking algorithm the problems need to be correlated to achieve a benefit compared to in the classical approach where they are done separately ~\cite{Tan2023}.

\subsection{Our Contribution}

We contribute to the theoretical understanding by giving the first insight into how these multitasking evolutionary algorithms function on submodular monotone functions with uniform linear constraints. We provide runtime bounds on multitasking Pareto optimization for solving multiple submodular monotone problems with these constraints. For each problem set we find the population size and use it to derive exact upper bounds on the runtime. We prove this for the most general case with $k$ problems. It's because of the correlation between the constraints, as well as the nature of the problems, that we can efficiently do so.

The runtime analysis shows the multitasking approach can have a superior runtime in some scenarios compared to the classical approach by at most factor of $k$ and at worst has an equal runtime. This is because the Pareto approach means that we can share the population and by finding the largest of the problems, we get all the lower order approximations for free. As the primary driver for the bounds is the size of the population, problems that are subsets of the most general problem, by sharing a bound for example, are trivial to derive.

To give insight into the dynamics of the multitasking approach, an experimental analysis was done applying the NP-hard max cover problem to multiple social graphs. We consider multiple problems and maximize them all classically as well as using our multitasking regime. Unlike in the theoretical case, we loosen our restrictions on the weights and allow them to vary between items. We show from this, the benefits of the multitasking approach in practice for constrained compute scenarios when the bounds are close together, similar to the runtime analysis comparison. For the uniform linear constraint case when the problems bounds are spread out and the number of generations is low, the multitasking approach gives worse approximations for the much smaller bounded problems. The middle of the range bounds were shown to perform well up until the largest generations. For the non-uniform constraints the multitasking approach gave roughly equivalent results. This indicates that in practice, multitasking GSEMO's effectiveness is highly dependent on the similarity and specifics of the problems.

The paper is structured as follows:
Section 2 outlines the preliminary definitions and notation. Section 3 gives the runtime analysis for the unit constraint cases, where we begin with the most general $k$ constraints case and discuss the applicability to other scenarios.  Section 3 concludes with comparisons between our evolutionary multitasking approach and the classical approach. Section 4 contains experimental results and analysis.

\section{Preliminaries}
This section focuses on describing various results and definitions relevant to our analysis. To begin we define what it means for a function to be submodular and monotone

\begin{definition} (Equivalent definitions of submodularity).
	Let $U$ be a finite ground set and $f: 2^U \to \mathbb{R}_0^+$ be a pseudo-boolean objective function. We have that:
\newline
	a) A function $f$ is submodular if for all $A \subseteq B \subseteq U$ and every $x \notin B$
	\begin{equation*}
		f(A \cup \{x\}) - f(A) \geq f(B \cup \{x\}) - f(B)
	\end{equation*}
	b) A function $f$ is submodular if for all $A \subseteq B \subseteq U$ and $a \notin B$
	\begin{equation*}
		f(B) \leq f(A) + \sum_{x \in B \backslash A} (f(A \cup \{x\})-f(A))
	\end{equation*}
    \label{def:sub}
\end{definition}

\begin{definition} (Monotonicity)
    A function $f$ is monotone if for all $A \subseteq B$
\begin{equation*}
		f(A) \leq f(B)
\end{equation*}
\end{definition}

To represent solutions of these functions we use a bit-string of length $n$, defined as $\{0,1\}^n$. Let each element in the ground set correspond to a bit in the bit-string as follows $X \subseteq U$ with $X=\{u_i \in U | x_i = 1, 1 \leq i \leq n \}$.

\begin{definition}
	The largest marginal gain is the increase of $f(x)$ caused by adding a single item (i.e. a single bit) and is defined as 
	\begin{equation*}
\delta_{i+1} = \max_{x \in OPT \setminus X_i} (f(X_i \cup \{x\}) - f(X_i))
\end{equation*}

	Where $OPT$ is an optimal solution and $X_i$ is a feasible solution. 
\end{definition}

A cost constraint is called a uniform cost constraint if it counts the number of $1$-bits in a solution, $c(x) = \sum_{i=1}^n x_i = |x|_1$. It is called a linear weighted uniform constraint if it has a weight applied to this count, $c(x) = a \cdot \sum_{i=1}^n x_i = a \cdot |x|_1$ where $a \in \mathbb{R}^+$. This weight $a$, can be normalized into the bound $B$ such that $c^*(x) = |x|_1$ and $B' = B/a$ if needed. We will keep them non-normalized to keep the results more practically useful and simplify further extensions looking into different constraints and costs. When these weights are $1$ we call them unit weights. We will only be considering uniform cost constraints in the runtime analysis and for simplicity use the term constraint to refer to the cost constraint function and bound together.

These constraints are special cases of more general knapsack constraints of the form $c(x) = \sum_{i=1}^n a_i x_i$.

\begin{definition}
	(General Problem Form). Given a submodular monotone function $f$, with a cost function $c: \{0,1\}^n \to \mathbb{R}^+$ that is a linear weighted uniform constraint and a bound $B \in \mathbb{R}^+$, we want to find
	\begin{equation*}
		\underset{x \in \{0,1\}^n}{\max} f(x) \text{ such that } c(x) \leq B
	\end{equation*}
\end{definition}

\begin{definition}
    (Multitasking Problem Form). Given a submodular monotone function $f: \{0,1\}^n \to \mathbb{R}$, $k$ cost functions $c_i: \{0,1\}^n \to \mathbb{R}$ for $i \in \{1, \dots, k\}$ that are linear weighted uniform constraints, and $k$ bounds $B_i \in \mathbb{R}^+$, the goal is to find a set of $k$ solutions $\{x_1, \dots, x_k\}$ such that for each $i \in \{1, \dots, k\}$:
    \begin{equation*}
        x_i \in \underset{x \in \{0,1\}^n}{\text{max}} f(x) \text{ such that } c_i(x) \leq B_i
    \end{equation*}
\end{definition}

This multitasking form can be represented as the problems defined by the triplet $(f, c_i, B_i)$ where $i \in \{1, \dots, k\}$ indexes each problem.

An example of a submodular monotone function is the maximum coverage problem. We have $G=(V,E)$ where $G$ is an undirected graph with $n=|V|$ nodes. Let every $u \in V$ have a corresponding cost $c_v(v)$ and set of nodes comprising of $v$ and their neighbors $N(v)$. Then we want to maximize the coverage given by the set of nodes $x$,

\begin{equation}
    \text{Coverage}(x) = \left|\bigcup_{i=1, \, x_i=1}^n N(v_i)\right|
    \label{eqn:coverage}
\end{equation}
such that the cost constraint $c(x)=\sum^n_{i=1} (c_v(v_i) \cdot x_i) \leq B$ holds.

As maximizing these functions with constraints is NP-hard, we need to define the concept of an approximation of the optimal solution. Let $OPT_i$ be an optimal feasible solution for problem $(f, c_i, B_i)$. Our goal is to compute for each $i \in \{1, \ldots, k\}$ a solution $x$ with $f(x) \geq \alpha \ OPT_i$ and $c_i(x) \leq B_i$.

By using a multi-objective algorithm, we can set the secondary objectives to maximize negative the number of $1$'s considering the weights of the given cost. This allows us to maximize $f$ while minimizing $c$, while also only including feasible solutions. We do this by introducing the concept of Pareto domination and a new objective function.

\begin{definition} (Pareto domination).
	Let $g=(g_1, \ldots, g_k) \colon \{0,1\}^n \rightarrow \mathds{R}^k$ be an objective function with $k$ objectives that should be maximized.
	Given two search points $x, y \in \{0,1\}^n$, we say that $x$ weakly dominates $y$ ($x \succeq y$) iff $g_i(x) \geq g_i(y)$, $1 \leq i \leq k.$ We say that $x$ dominates $y$ ($x \succ y$) iff $x \succeq y$ and $g_i(x)>g_i(y)$ for at least one $i \in \{1, \ldots, k\}$.
\end{definition}

The general problem form only contains one cost and one bound; however we investigate cases where there are multiple constraints and bounds. When there are $k$ different costs and bounds they are be represented in the form: $c_i(x) = a_i \cdot |x|_1 \leq B_i$, where $|x|_1$ is the count of $1$'s in a solution and $B_i$ is positive and $i \in \{1, \ldots, k\}$.

\begin{definition}
    The general multi-objective objective function GSEMO attempts to solve is
    \begin{equation*}
        g(x) = (g_1(x), g_2(x), \dots, g_{k+1}(x))
    \end{equation*}
    Where $g_1(x)$ is the primary submodular monotone function and $\{g_2(x), \dots, g_{k+1}(x)\}$ is the set of negative cost functions to be maximized.
\end{definition}

\begin{definition}
\label{classic:pri}
	The classical primary objective function is defined as:
	\begin{equation*}
		g_1(x) = 
		\begin{cases}
			f(x), & \text{if}\ x \in X \land c(x) \leq B \\
			-1,   & \text{else}                          
		\end{cases}
	\end{equation*}
\end{definition}
In the classical formalization, to ensure that the cost is minimized we have that $g_2(x) = -c(x)$ as this is equivalent to maximizing the number of $0$-bits in a solution.

\begin{definition}
\label{pri:multitask}
	The most general multitasking primary objective function shared by each $i \in \{1, \dots, k\}$ problems is defined as:
	\begin{equation*}
		g_1(x) = 
		\begin{cases}
			f(x), & \text{if }\ x \in X \land \exists i \in \{1, \ldots, k\} \text{ with} \ c_i(x) \leq B_i \\
			-1,   & \text{else}     
		\end{cases}
	\end{equation*}
\end{definition}

In our multitasking formulation, we consider multiple problems where for $k$ problems we have $k$ secondary functions defined as the negative of the individual cost functions, such that $g_{i+1}(x) = -c_i(x)$ for $i \in \{1, \dots, k\}$. This means that overall we have $k+1$ different objective functions, where the primary objective function, the submodular monotone function, is fully shared between the problems and only needs to be evaluated once per generation.

\begin{algorithm}[tb]
    \raggedright
    \caption{GSEMO}
    \label{alg:gsemo}
    \begin{algorithmic}[1]
        \STATE Choose $x \in \{0,1\}^n$ uniformly at random;
        \STATE $P \gets \{x\}$;
        \STATE $t \gets 0$;
        \REPEAT
            \STATE Choose $x \in P$ uniformly at random;
\STATE Create $y$ by flipping each bit of $x$ independently with probability $1/n$;
\IF{$\nexists w \in P : w \succ y$}
                \STATE $P \gets (P \setminus \{z \in P \mid y \succeq z\}) \cup \{y\}$;
            \ENDIF
            \STATE $t \gets t + 1$;
        \UNTIL{$t \geq t_{max}$}
        \RETURN $P$;
    \end{algorithmic}
\end{algorithm}

The general method for each proof is to first setup the problem by defining the constraints and the bounds to give us $B_{max}$, the number of possible $1$-bits a feasible solution can have also referred to as the maximum bound. We then derive the population upper bound $U$. Then we can find the time needed until the first Pareto optimal individual $0^n$, the individual with only $0$ bits, is in the population. Finally, we use induction to find the time until the rest of the Pareto front is found. We measure runtime as the number of evaluations of the function $f$. The expected time then is the expected number of fitness evaluations of the function $f$. We primarily give the results in terms of $U$ rather than $B_{max}$ as it makes for a more natural representation of that the upper bound is and in our case they are asymptotically the same.

If the constraints were more costly, then we would need to take them into account, but we assume that the function evaluation is the most expensive operation. Then for large enough generations, the cost of the comparisons is comparably small. We also do not collapse the constraints down when this is possible both because the experimental analysis shows that in practice this equivalence does not occur even for a large number of generations and also because it allows for knapsack constraints, which are considered in the experimental section as an extension but not in the theoretical analysis section.

\section{Theoretical analysis}

In this section we consider optimizing multiple monotone submodular problems with uniform linear constraints. We begin with the general case where we have $k$ different constraints.

Let $f$ be a monotone submodular function, $c_i(x) = a_i \cdot |x|_1$ be the cost functions, and $B_i$ be their constraint bounds where $i \in \{1,\dots, k\}$. Thus, in this case we have the following multi-objective problem:
\begin{equation*}
	\label{k_static}
	\underset{x}{\max} \ (g_1(x), -a_1 |x|_1, -a_2 |x|_1, \ldots, -a_k |x|_1),
\end{equation*}

\begin{theorem}
\label{thm:1}
	The population size of GSEMO solving the monotone submodular problems $(f, c_i, B_i)$, for $i \in \{1, \dots, k\}$, is upper bounded by $U = \min(\max_{1 \le i \le k}\left(\floor*{\frac{B_i}{a_i}}\right), n) + 1$.
\end{theorem}
\begin{proof}
	We have constraints of the form $a_i |x|_1 \leq B_i$, which implies that the maximum number of bits, $|x|_1$, is $\max_{1 \le i \le k}\left(\floor*{\frac{B_i}{a_i}}\right)$. As $\max_{1 \le i \le k}\left(\floor*{\frac{B_i}{a_i}}\right)$ can be larger than $n$, the maximum value of $|x|_1$, $B_{max}$, is the minimum of either $n$ or $\max_{1 \le i \le k}\left(\floor*{\frac{B_i}{a_i}}\right)$.
	
	To show that there is at most only one individual for each $i \in \{0, $\ldots$, B_{max}\}$, let $x$ be the parent and $y$ be the offspring with the same number of elements, $|x|_1=|y|_1$. Because they have the same number of elements, all the secondary objectives $-a_i |x|_1=-a_i |y|_1$ are equal which fulfills the first part of the domination definition. If $g_1(y) \geq g_1(x)$, then $y$ replaces $x$. Otherwise $y$ is discarded.
	 
	Thus, the upper bound for the population size is $$U = B_{max} + 1 = \min(\max_{1 \le i \le k}\left(\floor*{\frac{B_i}{a_i}}\right), n) + 1$$
\end{proof}

The above proof shows why the Pareto front won't explode exponentially, the secondary objectives are linear combinations of each other (the bit count) and thus collapses nicely into a single dimension.

\begin{theorem}
\label{thm:2}
	 The expected time until GSEMO has obtained, for each monotone submodular problem $(f, c_i, B_i)$ where $i \in \{1, \dots, k\}$, a $\left(1-\frac{1}{e}\right)$-approximation is $\mathcal{O}(Un(\log n + U))$.
\end{theorem}

\begin{proof}
	\label{general_proof}
	We first analyze the time until the search point $0^n$ has been obtained for the first time. This search point has the smallest cost with respect to all constraint functions and stays in the population once added. Furthermore, the search point $0^n$ dominates any infeasible solution and infeasible solutions are therefore rejected once the search point $0^n$ has been added to the population.
    
	Consider $\ell = |y|_1$ to be the individual with the smallest number of $1$-bits in the population $P$. Due to the dominance in the Pareto front, it can never increase over the run of the algorithm. In each step, a single bit flip alone occurs with a probability of at least
    $$
    \frac{1}{n}\left( 1 - \frac{1}{n} \right)^{n-1} \geq \frac{1}{e}
    $$
    As there are $n$ bits per element and $y$ has $\ell$ $1$'s, the probability that $y$ decreases by 1 is $\ell / en$. The probability that GSEMO selects $y$ is at least $1/U$. Using the methods of fitness-based partitioning, the upper bound for the expected runtime until $0^n$ is found is therefore 
	
	\begin{equation}
		\label{eqn:first}
		\sum^n_{\ell=1} \left(\frac{\ell}{enU}\right)^{-1} = \mathcal{O}(Un \log n)
	\end{equation}
	
	We show that at the end of the runtime there exists a $\left(1-\frac{1}{e}\right)$-approximation for every $b \in \{0, $\ldots$, B_{max}\}$, which by definition includes each problem.

    For every $i \in \{0, $\ldots$, b\}$ and some $b\in \{0, $\ldots$, B_{max}\}$, we want to show that 
	\begin{equation}
		\label{eqn:assum}
		f(X_i) \geq \left(1-\left(1-\frac{1}{b}\right)^i\right) \cdot f(OPT_b)
	\end{equation}
	holds, such that $OPT_b$ is a feasible optimal solution for each $b$ which includes all of the constraints. Once $0^n$ is found, this holds for $i=0$. Assume the inductive hypothesis, that this equation does hold for a given $X_i$.
	
	We need to find a bound on the optimal solution and by using the second definition of submodularity along with noting that an optimal feasibly solution $OPT_b$ contains at most $b$ elements, the function's monotonicity and the definition of the largest marginal increase
	\begin{equation*}
		\begin{split}
			f(OPT_b) & \leq f(X_i \cup OPT_b) \\
			& \leq f(X_i) +\sum_{x \in OPT_b\backslash X_i} (f(X_i \cup \{x\})-f(X_i)) \\
			& \leq f(X_i) + |OPT_b| \delta_{i+1} = f(X_i) + b \delta_{i+1}
		\end{split}
	\end{equation*}
	Rearranging this gives that the next marginal gain is bound by
    $\delta_{i+1} \geq \frac{1}{b} (f(OPT_b)-f(X_i))$.

	Again, using the monotonic and submodular nature of $f$
	\begin{equation*}
		f(X_{i+1}) \geq f(X_i) + \delta_{i+1} \geq f(X_i) + \frac{1}{b} (f(OPT_b)-f(X_i))
	\end{equation*}

	Considering the assumptions about $f(X_i)$ from Equation \ref{eqn:assum} we get
	\begin{equation*}
		\begin{split}
			f(X_{i+1}) & \geq f(X_i)\left(1-\frac{1}{b}\right)+\frac{1}{b}\cdot f(OPT_b)\\
			& \geq \left(1-\left(1-\frac{1}{b}\right)^{i+1} \right) \cdot f(OPT_b) \\
		\end{split}
	\end{equation*}
	When the final solution is found, $i=b$, we have 
    $$f(X_b) \geq \left(1-\left(1-\frac{1}{b}\right)^{b}\right) \cdot f(OPT_b) \geq \left(1-\frac{1}{e}\right) \cdot f(OPT_b)$$ for any $b\in \{0, \dots, B_{max} \}$
    If $b = B_{max}$, then all lower-order approximations must exist in the population as we have shown that approximations have been found for all $b$. Therefore approximations for all the monotone submodular functions have been found. 
	
	The time to achieve the final $\left(1-\frac{1}{e}\right)$-approximation for $B_{max}$ depends on the number of steps required to progress. For each of these steps from $i$ to $i+1$, we have the probability of selecting the best individual as at least $1/U$ and the probability of flipping the right bit as at least $1/(en)$ as we only need the element with the largest marginal increase to be flipped. We repeat each step at most $U$ times so we get the runtime upper bound to be
	
	\begin{equation}
		\label{eqn:second}
		\sum^{U}_{i=0} \left(\frac{1}{Uen}\right)^{-1} = \mathcal{O}(n \ U^2)
	\end{equation}
	
	Putting the two upper bounds from Equations \ref{eqn:first} and \ref{eqn:second} together, we get $ \mathcal{O}(U n \log n) + (n \ U^2) = \mathcal{O}(Un (\log n + U)$. 
\end{proof}

Often, it's useful to start the algorithm from a known initial search point such as $0^n$. From the previous proof we can derive the runtime for this trivially.
\begin{corollary}
	If $0^n$ is the initial search point rather than a random one, the expected time until GSEMO has obtained, for each monotone submodular problem $(f, c_i, B_i)$, where $i \in \{1, \dots, k\}$, a $\left(1-\frac{1}{e}\right)$-approximation is $\mathcal{O}(U^2n)$.
\end{corollary}

We also know that the upper bound on the population size is $n$ as our population size never grows exponentially as is sometimes seen in \gsemo{} for more complex scenarios, giving us the following corollaries that show the direct links between the classical and multitasking approaches.

\begin{corollary}
    The expected time until GSEMO has obtained, for each monotone submodular problem $(f, c_i, B_i)$, where $i \in \{1, \dots, k\}$, a $\left(1-\frac{1}{e}\right)$-approximation is $\mathcal{O}(n^2(\log n + U))$.
\end{corollary}

\begin{corollary}
    If $0^n$ is the initial search point rather than a random one, then the expected time until GSEMO has obtained, for each monotone submodular problem $(f, c_i, B_i)$, where $i \in \{1, \dots, k\}$, a $\left(1-\frac{1}{e}\right)$-approximation is $\mathcal{O}(Un^2)$.
\end{corollary}

\subsection{Subsets of the general case}

The proofs show that the parameters that play dominant roles in determining the runtime bounds are the upper bound on the population size $U$ and the maximum number of $1$-bits that can be selected $B_{max}$. This is because the solutions can't have more bits then $B_{max}$ and so we find approximations for every possible $1$-count up to $B_{max}$ in $U$. Subsets of the general case that still have uniform constraints are therefore trivial, as all that is needed to derive $U$ and thus runtime bounds using the proof of Theorem \ref{general_proof}, is to know $B_{max}$.

This is because the population size proofs rely on $B_{max}$, the sharing of the primary objective function $f$ and the uniformity of the constraints meaning any combination of constrains will trivially fulfill the secondary objectives.

\subsection{Comparison between methods}

If we have $k$ problems, the expected runtime to find a $(1 - \frac{1}{e})$-approximation for every problem in the multitasking scenario is 
\[
T_{m} = \mathcal{O}(U n(\log n + U)),
\]
where $U = \max_{1 \le i \le k}(U_i)$ is the largest relevant population size across all problems.

In the classical approach, each problem is solved independently, giving $k$ separate runtimes:
\[
T_c = \mathcal{O}\left( \sum_{i=1}^k U_i n(\log n + U_i) \right).
\]

We only have upper bounds, which limits our comparison. But, since for any $i$, $U \geq U_i$, the multitasking approach is never going to be asymptotically worse than the classic approach in terms of upper bounds for solving all of the $k$ problems. The speed up is entirely dependent on the similarity between $U_i$'s. The best speed up comes when $U_1 \approx U_2 \approx \dots \approx U_k$, as in this case the sum in $T_c$ will approach $k$ times the upper bound on $T_m$. In the worst case however, where $U_k \gg \sum^{k-1}_{i=1}$, then $T_c \approx T_m$ and there is no worst case speed up at all.

These results are directly related to the Pareto multitasking approach that allows for sharing the population of the Pareto front. The subsequent $\left(1-\frac{1}{e}\right)$-approximations for every possible bit count up to the largest problems constraint ratio are then only found once and can remain in the population.

\section{Experimental investigations}

To experimentally investigate the multitasking approach, we consider the max cover problem from Equation \ref{eqn:coverage}, with unit weights, uniformly distributed random weights and node degree based weights. To allow for a more controlled comparison, we begin with an initial solution $0^n$.  We test on the social graphs ca-GrQc, Erdos992, ca-HepPh, ca-AstroPh, ca-CondMat, graphs, taken from the Network Repository \cite{DBLP:conf/aaai/RossiA15}. We test over 100000, 500000, 1000000 generations, each run 30 times over 5 different constraints. We denote the mean and standard deviation of the approaches by $\mu$ and $\sigma$, with the subscript $c$ for the classical algorithm and $m$ for the multitasking algorithm.

We assume primary function evaluations of the submodular function dominate the overall compute. Pareto dominance and constraint checks scale efficiently in optimized array implementations and are computationally trivial. Scaling the generation bound by $k$ for the multitasking approach thus gives a fair comparison by comparing between the same number of primary function evaluations per problem as the multitasking approach shares the evaluation function.

We consider three scenarios: unit constraints (aligned with our theoretical analysis), randomly weighted linear constraints, and degree-weighted linear constraints. The unit constraint case is a special case of our theoretical analysis where we normalize the weight that is shared between all items. The latter two cases are more general settings than our theoretical analysis, where the weights can vary between different items. The weights here are given in terms of $c_v$, the vertex costs, however they correspond directly with the weights $a_i$ from our prior analysis. These break the proven upper bound on the Pareo front as the constraints are no longer linear combinations of each other, thus for these results, the analysis and small Pareto fronts in the theoretical section do not apply. Rather, they are included for insight into practical dynamics.

For the unit constraint case $c_v = 1$, we start with a trivial bound of $1$ all the way up to large non-trivial bounds. We pick the trivial bound to show that there are statistically significant differences between the two approaches even in such a situation. For the randomly weighted linear constraint case, we scale the bounds by $100$ and sample the weights for each item $u_i \sim \mathcal{U}(50, 150)$ such that $c_v = \ceil{u_i}$, normalizing back to the original bound once complete. If this was done as sampling from $\mathcal{U}(0.50, 1.50)$ and using the unscaled bounds, we would face rounding errors. This scenario allows us to to investigate how multitasking works when using uncorrelated weights. For the degree-weighted linear constraint case, we use the bounds from the linear constraint case but weigh the items based on their node degree. That is, $c_v = \text{deg}(v)$.

Kruskal-Wallis tests are used to check for significance between the classical and multitasking cases. These are rank based tests and don't rely directly on the mean or median, with $p \leq 0.05$ chosen as determining significance.

\subsection{Results}

\subsubsection{Unit constraints}

From Table \ref{tab:combined_weights} focusing first on the unit constraint case, we can see patterns emerge that show that experimentally the multitasking approach isn't always superior. The lower bounds tend to do worse and occasionally the middle bounds tend to be worse at the largest number of generations. This is notable as we can see these results are consistent amount all the graphs, which have different sizes, density and topology. Even for the trivial single item constraint, the multitasking approach is significantly worse. Interestingly, multitasking GSEMO does better for the non-unit weight cases.

\begin{figure}[t]
    \centering
    \subfloat[Comparison between the classical and multitasking approaches for two different sets of generations looking at the problem with $B_{max} = 415$.\label{fig:good}]{%
        \includegraphics[width=0.48\textwidth]{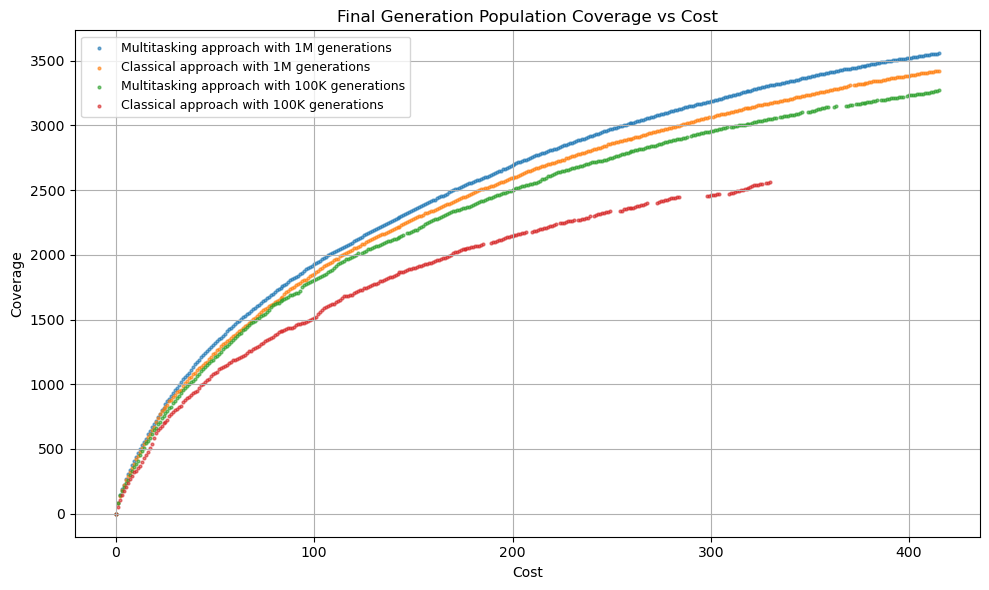}%
    }\hfill
    \subfloat[Mean coverage of 4 problems with problem bounds $50$, $100$, $200$ and $300$ optimized over $200$k generations.\label{fig:bad}]{%
        \includegraphics[width=0.48\textwidth]{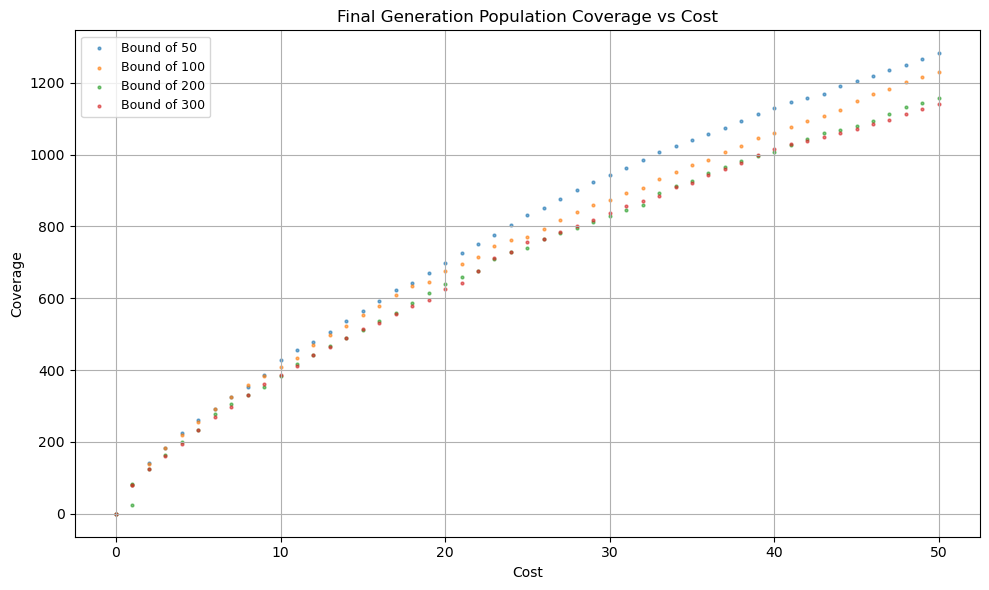}%
    }
    \caption{Performance and coverage comparisons. (a) Generation comparison for $B_{max} = 415$. (b) Solution coverage over 200k generations.}
    \label{fig:combined_figures}
\end{figure}

Figure \ref{fig:good} shows how for the largest problem, the multitasking approach gives a vastly improved result. This is because it's able to share the population and function evaluations with the rest of the problems, giving it effectively more compute. For the smaller constraints, the larger population size of the multitasking approach causes a disadvantage, as the search space is much larger meaning it doesn't spend as much improving the solutions for these smaller constraints. This means that the classical approach can pick more optimal solutions much easier as it focuses on the reduced area while the multitasking approach spends more time improving other problems. To visualize this, Figure \ref{fig:bad} shows several uniform linear constraint problems optimized using the classical approach, we can see that the smaller problems give better coverage for the lowest of the bounds.

When the multitasking approach is worse at the largest number of generations, it's likely due to the classical approach being able to focus more on the constrained search points due to the smaller population size as with the small bounds. In this case, the sharing of populations means that in the initial stages the penalty from this larger search space is small enough that the advantage of sharing the function evaluations outweighs it. But as the approximations improve, the sharing of the function evaluations becomes less valuable as the multitasking algorithm spends time searching for, and improving, solutions that aren't feasible for these middle size problems.

The largest bound is always better optimized using the multitasking algorithm. This is expected, as essentially, we are just solving it but with $5$ times the bound, as all lower problems are sharing this compute budget. The second largest problem also always does significantly better with the multitasking algorithm, which gives weight to the idea that the closer together the bounds are, the more effective the multitasking approach is. Likely due to the penalty from the larger search space being restricted by the cost savings of sharing function evaluations.

\subsubsection{General linear constraints}

For the randomly weighted linear constraint problems noted in Table \ref{tab:combined_weights}, we see similar results due to similar behind the scenes mechanics of the algorithm even with the shared weights. Notably though, we see that the multitasking approach is always more effective for the middle bound than in the unit weight case rather than just early on in the optimization process. This indicates that the penalty of the larger search space isn't as important once we have picked high-coverage items.

The results in Table \ref{tab:combined_weights} for the degree weighted linear constraint problems however, show that it was rare for the multitasking approach to be worse, instead there was no significant difference between the two approaches for the lower bounds. Notably, the coverage for both algorithms were much lower than the results for the other two constraint regimes. Here the benefit of sharing the population balances the drawback of the reduced search space for the larger problems. This is likely due to the fact that these are more difficult problems and GSEMO must be more careful picking optimal and feasible items while multitasking GSEMO is able to exploit its ability to have a higher effective number of generations to make these choices.

These results are related to the justification for the efficacy of sliding window approaches, which reduces this search space allowing the algorithm to focus on optimizing certain values more effectively. This indicates that the combined sliding window approach and multitasking approach will potentially be able to exploit the advantages of each other highly effectively.

\begin{table}[htbp]
\centering
\caption{Results for the three sets of experiments. +* indicates Multitasking was significantly better, -* indicates Classical was better, = indicates no difference.}
\label{tab:combined_weights}
\tiny
\begin{adjustbox}{max totalheight=\textheight, keepaspectratio, max width=\textwidth}
\begin{tabular}{|l|r|r|ccccc|ccccc|ccccc|}
\toprule
Graph & Bound & Gen & \multicolumn{5}{c|}{Unit} & \multicolumn{5}{c|}{Random Linear} & \multicolumn{5}{c|}{Degree Linear} \\
 & & & $\text{Average}_{\text{c}}$ & $\text{Std}_{\text{c}}$ & $\text{Average}_{\text{m}}$ & $\text{Std}_{\text{m}}$ & Stat & $\text{Average}_{\text{c}}$ & $\text{Std}_{\text{c}}$ & $\text{Average}_{\text{m}}$ & $\text{Std}_{\text{m}}$ & Stat & $\text{Average}_{\text{c}}$ & $\text{Std}_{\text{c}}$ & $\text{Average}_{\text{m}}$ & $\text{Std}_{\text{m}}$ & Stat \\
\midrule
\multirow{20}{*}{\rotatebox{90}{Erdos992}}
 & 1 & 100000 & \textbf{62.0} & \textbf{0.0} & 55.7 & 5.3 & -* & \textbf{61.6} & \textbf{1.3} & 52.9 & 7.4 & -* & 2.0 & 0.0 & 2.0 & 0.0 & = \\
 & 1 & 200000 & \textbf{62.0} & \textbf{0.0} & 59.4 & 4.0 & -* & \textbf{62.0} & \textbf{0.0} & 55.8 & 6.5 & -* & 2.0 & 0.0 & 2.0 & 0.0 & = \\
 & 1 & 500000 & \textbf{62.0} & \textbf{0.0} & 61.3 & 1.3 & -* & \textbf{62.0} & \textbf{0.0} & 59.0 & 3.9 & -* & 2.0 & 0.0 & 2.0 & 0.0 & = \\
 & 1 & 1000000 & \textbf{62.0} & \textbf{0.0} & 61.5 & 1.0 & -* & \textbf{62.0} & \textbf{0.0} & 60.9 & 2.4 & -* & 2.0 & 0.0 & 2.0 & 0.0 & = \\
 & 12 & 100000 & \textbf{589.2} & \textbf{5.2} & 527.4 & 15.7 & -* & \textbf{676.5} & \textbf{21.7} & 619.4 & 22.2 & -* & 24.0 & 0.0 & 24.0 & 0.0 & = \\
 & 12 & 200000 & \textbf{599.2} & \textbf{2.6} & 554.4 & 16.7 & -* & \textbf{740.7} & \textbf{14.3} & 674.7 & 19.4 & -* & 24.0 & 0.0 & 24.0 & 0.0 & = \\
 & 12 & 500000 & \textbf{603.3} & \textbf{1.2} & 589.2 & 5.2 & -* & \textbf{788.7} & \textbf{7.5} & 750.0 & 15.3 & -* & 24.0 & 0.0 & 24.0 & 0.0 & = \\
 & 12 & 1000000 & \textbf{603.9} & \textbf{0.3} & 599.0 & 3.4 & -* & \textbf{806.1} & \textbf{3.9} & 784.7 & 10.5 & -* & 24.0 & 0.0 & 24.0 & 0.0 & = \\
 & 78 & 100000 & 1896.0 & 41.0 & \textbf{1985.5} & \textbf{36.0} & +* & 1754.5 & 38.3 & \textbf{2027.3} & \textbf{33.0} & +* & 156.0 & 0.0 & 156.0 & 0.0 & = \\
 & 78 & 200000 & 2141.9 & 28.7 & \textbf{2169.7} & \textbf{26.2} & +* & 1990.9 & 36.7 & \textbf{2218.5} & \textbf{26.0} & +* & 156.0 & 0.0 & 156.0 & 0.0 & = \\
 & 78 & 500000 & 2378.1 & 15.4 & 2369.2 & 17.7 & = & 2283.7 & 23.1 & \textbf{2455.8} & \textbf{19.7} & +* & 156.0 & 0.0 & 156.0 & 0.0 & = \\
 & 78 & 1000000 & 2453.1 & 4.8 & 2450.4 & 5.1 & = & 2478.4 & 16.3 & \textbf{2602.0} & \textbf{12.7} & +* & 156.0 & 0.0 & 156.0 & 0.0 & = \\
 & 305 & 100000 & 2984.7 & 55.4 & \textbf{3780.8} & \textbf{33.5} & +* & 2662.0 & 55.7 & \textbf{3605.6} & \textbf{32.6} & +* & 572.1 & 2.8 & \textbf{606.4} & \textbf{1.2} & +* \\
 & 305 & 200000 & 3391.6 & 38.4 & \textbf{4126.8} & \textbf{24.7} & +* & 3129.6 & 40.3 & \textbf{3950.7} & \textbf{21.7} & +* & 599.1 & 1.9 & \textbf{610.0} & \textbf{0.0} & +* \\
 & 305 & 500000 & 3927.1 & 28.5 & \textbf{4533.3} & \textbf{10.5} & +* & 3663.8 & 31.7 & \textbf{4350.4} & \textbf{16.4} & +* & 610.0 & 0.0 & 610.0 & 0.0 & = \\
 & 305 & 1000000 & 4306.0 & 17.5 & \textbf{4727.5} & \textbf{7.1} & +* & 4021.0 & 25.2 & \textbf{4589.8} & \textbf{10.2} & +* & 610.0 & 0.0 & 610.0 & 0.0 & = \\
 & 610 & 100000 & 3139.3 & 59.0 & \textbf{4452.8} & \textbf{30.5} & +* & 2646.6 & 67.5 & \textbf{4026.7} & \textbf{36.8} & +* & 871.1 & 5.6 & \textbf{1008.5} & \textbf{3.4} & +* \\
 & 610 & 200000 & 3759.9 & 55.9 & \textbf{4811.4} & \textbf{16.4} & +* & 3269.8 & 50.8 & \textbf{4516.3} & \textbf{39.3} & +* & 937.8 & 4.9 & \textbf{1039.8} & \textbf{2.4} & +* \\
 & 610 & 500000 & 4465.7 & 27.0 & \textbf{5053.2} & \textbf{7.1} & +* & 4027.1 & 39.6 & \textbf{4904.9} & \textbf{13.3} & +* & 1008.8 & 3.8 & \textbf{1049.7} & \textbf{0.6} & +* \\
 & 610 & 1000000 & 4807.8 & 16.0 & \textbf{5092.6} & \textbf{1.2} & +* & 4521.2 & 26.3 & \textbf{5043.4} & \textbf{5.5} & +* & 1039.3 & 2.0 & \textbf{1050.0} & \textbf{0.0} & +* \\
\midrule\midrule
\multirow{20}{*}{\rotatebox{90}{ca-AstroPh}}
 & 1 & 100000 & \textbf{497.3} & \textbf{23.5} & 290.4 & 76.9 & -* & \textbf{401.1} & \textbf{33.1} & 226.4 & 92.4 & -* & 2.0 & 0.0 & 2.0 & 0.0 & = \\
 & 1 & 200000 & \textbf{502.4} & \textbf{14.1} & 305.0 & 75.6 & -* & \textbf{419.2} & \textbf{16.5} & 250.6 & 94.8 & -* & 2.0 & 0.0 & 2.0 & 0.0 & = \\
 & 1 & 500000 & \textbf{505.0} & \textbf{0.0} & 353.3 & 66.8 & -* & \textbf{425.8} & \textbf{7.6} & 281.2 & 83.6 & -* & 2.0 & 0.0 & 2.0 & 0.0 & = \\
 & 1 & 1000000 & \textbf{505.0} & \textbf{0.0} & 374.8 & 57.7 & -* & \textbf{427.7} & \textbf{1.6} & 302.0 & 79.6 & -* & 2.0 & 0.0 & 2.0 & 0.0 & = \\
 & 14 & 100000 & \textbf{2611.8} & \textbf{83.9} & 2091.5 & 92.8 & -* & \textbf{2578.5} & \textbf{84.7} & 2297.9 & 90.8 & -* & \textbf{28.0} & \textbf{0.0} & 27.8 & 0.4 & -* \\
 & 14 & 200000 & \textbf{2769.3} & \textbf{64.8} & 2199.7 & 73.3 & -* & \textbf{2767.3} & \textbf{61.8} & 2460.9 & 92.2 & -* & 28.0 & 0.0 & 28.0 & 0.0 & = \\
 & 14 & 500000 & \textbf{2917.6} & \textbf{34.6} & 2355.9 & 96.3 & -* & \textbf{2993.5} & \textbf{42.1} & 2645.4 & 84.1 & -* & 28.0 & 0.0 & 28.0 & 0.0 & = \\
 & 14 & 1000000 & \textbf{2963.8} & \textbf{11.1} & 2508.3 & 81.1 & -* & \textbf{3142.4} & \textbf{34.7} & 2779.3 & 67.4 & -* & 28.0 & 0.0 & 28.0 & 0.0 & = \\
 & 133 & 100000 & 6470.5 & 73.6 & \textbf{6649.0} & \textbf{65.7} & +* & 6216.0 & 69.7 & \textbf{6854.5} & \textbf{76.4} & +* & \textbf{230.6} & \textbf{1.4} & 218.2 & 1.6 & -* \\
 & 133 & 200000 & 6943.1 & 69.3 & 6954.8 & 52.0 & = & 6719.1 & 58.8 & \textbf{7200.2} & \textbf{70.6} & +* & \textbf{246.8} & \textbf{1.4} & 232.0 & 1.5 & -* \\
 & 133 & 500000 & \textbf{7529.2} & \textbf{49.1} & 7341.3 & 47.4 & -* & 7336.5 & 51.9 & \textbf{7651.5} & \textbf{47.0} & +* & \textbf{264.5} & \textbf{0.9} & 251.5 & 1.3 & -* \\
 & 133 & 1000000 & \textbf{7938.7} & \textbf{46.9} & 7648.2 & 38.3 & -* & 7824.9 & 51.5 & \textbf{7985.1} & \textbf{52.1} & +* & \textbf{266.0} & \textbf{0.0} & 263.8 & 1.0 & -* \\
 & 895 & 100000 & 9510.3 & 94.6 & \textbf{12326.4} & \textbf{55.2} & +* & 8179.4 & 133.8 & \textbf{11439.6} & \textbf{102.1} & +* & 1098.8 & 4.0 & \textbf{1190.5} & \textbf{4.6} & +* \\
 & 895 & 200000 & 10911.2 & 92.8 & \textbf{12907.8} & \textbf{39.9} & +* & 9544.1 & 129.5 & \textbf{12499.1} & \textbf{44.6} & +* & 1156.4 & 2.7 & \textbf{1260.0} & \textbf{3.4} & +* \\
 & 895 & 500000 & 12340.1 & 50.3 & \textbf{13553.5} & \textbf{32.8} & +* & 11398.5 & 86.6 & \textbf{13353.4} & \textbf{39.0} & +* & 1245.2 & 3.1 & \textbf{1362.1} & \textbf{3.4} & +* \\
 & 895 & 1000000 & 13001.5 & 36.0 & \textbf{14036.7} & \textbf{28.4} & +* & 12525.7 & 46.4 & \textbf{13916.8} & \textbf{34.6} & +* & 1321.3 & 2.5 & \textbf{1447.1} & \textbf{3.8} & +* \\
 & 1790 & 100000 & 9466.3 & 104.3 & \textbf{12774.6} & \textbf{84.1} & +* & 8236.7 & 137.9 & \textbf{11439.6} & \textbf{102.1} & +* & 1986.9 & 6.1 & \textbf{2195.2} & \textbf{5.4} & +* \\
 & 1790 & 200000 & 10887.1 & 91.0 & \textbf{14112.3} & \textbf{54.4} & +* & 9582.6 & 112.7 & \textbf{12802.2} & \textbf{80.0} & +* & 2066.0 & 4.4 & \textbf{2313.7} & \textbf{4.6} & +* \\
 & 1790 & 500000 & 12737.1 & 79.3 & \textbf{15489.5} & \textbf{35.5} & +* & 11408.2 & 88.5 & \textbf{14508.3} & \textbf{69.0} & +* & 2195.9 & 4.8 & \textbf{2486.0} & \textbf{4.4} & +* \\
 & 1790 & 1000000 & 14092.2 & 63.7 & \textbf{16078.6} & \textbf{22.4} & +* & 12771.3 & 94.7 & \textbf{15574.7} & \textbf{36.1} & +* & 2312.5 & 4.4 & \textbf{2629.5} & \textbf{4.6} & +* \\
\midrule\midrule
\multirow{20}{*}{\rotatebox{90}{ca-CondMat}}
 & 1 & 100000 & \textbf{273.8} & \textbf{16.4} & 126.3 & 49.6 & -* & \textbf{204.2} & \textbf{47.0} & 83.5 & 29.6 & -* & 2.0 & 0.0 & 2.0 & 0.0 & = \\
 & 1 & 200000 & \textbf{280.0} & \textbf{0.0} & 138.8 & 54.4 & -* & \textbf{232.6} & \textbf{35.4} & 99.4 & 37.9 & -* & 2.0 & 0.0 & 2.0 & 0.0 & = \\
 & 1 & 500000 & \textbf{280.0} & \textbf{0.0} & 155.9 & 55.0 & -* & \textbf{249.6} & \textbf{12.9} & 109.3 & 33.5 & -* & 2.0 & 0.0 & 2.0 & 0.0 & = \\
 & 1 & 1000000 & \textbf{280.0} & \textbf{0.0} & 185.5 & 56.1 & -* & \textbf{253.0} & \textbf{0.0} & 119.1 & 30.6 & -* & 2.0 & 0.0 & 2.0 & 0.0 & = \\
 & 14 & 100000 & \textbf{1513.8} & \textbf{51.4} & 1046.2 & 63.4 & -* & \textbf{1454.7} & \textbf{73.7} & 1263.1 & 99.7 & -* & 28.0 & 0.0 & 27.9 & 0.3 & = \\
 & 14 & 200000 & \textbf{1656.5} & \textbf{47.7} & 1156.4 & 71.8 & -* & \textbf{1644.9} & \textbf{57.0} & 1366.8 & 83.0 & -* & 28.0 & 0.0 & 28.0 & 0.0 & = \\
 & 14 & 500000 & \textbf{1800.5} & \textbf{20.1} & 1300.4 & 69.3 & -* & \textbf{1839.3} & \textbf{47.8} & 1490.6 & 66.3 & -* & 28.0 & 0.0 & 28.0 & 0.0 & = \\
 & 14 & 1000000 & \textbf{1844.5} & \textbf{8.2} & 1424.4 & 68.2 & -* & \textbf{1980.1} & \textbf{28.9} & 1616.6 & 67.7 & -* & 28.0 & 0.0 & 28.0 & 0.0 & = \\
 & 146 & 100000 & 4359.6 & 73.6 & \textbf{4620.1} & \textbf{65.3} & +* & 4087.7 & 84.0 & \textbf{4778.7} & \textbf{69.6} & +* & \textbf{261.2} & \textbf{1.2} & 246.8 & 1.8 & -* \\
 & 146 & 200000 & 4899.8 & 70.5 & 4916.1 & 69.8 & = & 4639.5 & 70.9 & \textbf{5141.5} & \textbf{53.8} & +* & \textbf{277.9} & \textbf{1.4} & 260.5 & 1.7 & -* \\
 & 146 & 500000 & \textbf{5578.7} & \textbf{38.5} & 5316.3 & 55.9 & -* & 5329.2 & 54.9 & \textbf{5637.1} & \textbf{46.5} & +* & \textbf{292.0} & \textbf{0.2} & 281.1 & 1.5 & -* \\
 & 146 & 1000000 & \textbf{6108.6} & \textbf{42.2} & 5653.1 & 41.9 & -* & 5838.1 & 41.7 & \textbf{5988.4} & \textbf{48.7} & +* & \textbf{292.0} & \textbf{0.0} & 291.4 & 0.7 & -* \\
 & 1068 & 100000 & 7190.8 & 161.0 & \textbf{11307.9} & \textbf{98.4} & +* & 5775.2 & 162.1 & \textbf{9683.0} & \textbf{112.5} & +* & 1353.5 & 3.9 & \textbf{1467.2} & \textbf{2.8} & +* \\
 & 1068 & 200000 & 8971.2 & 159.4 & \textbf{12231.5} & \textbf{77.7} & +* & 7319.3 & 187.8 & \textbf{11568.4} & \textbf{72.1} & +* & 1422.2 & 4.0 & \textbf{1545.4} & \textbf{3.6} & +* \\
 & 1068 & 500000 & 11349.4 & 74.2 & \textbf{13194.6} & \textbf{51.4} & +* & 9679.1 & 160.6 & \textbf{12941.2} & \textbf{53.1} & +* & 1525.0 & 2.4 & \textbf{1662.3} & \textbf{3.9} & +* \\
 & 1068 & 1000000 & 12362.2 & 61.3 & \textbf{13936.6} & \textbf{35.8} & +* & 11578.5 & 84.1 & \textbf{13761.1} & \textbf{37.0} & +* & 1616.1 & 3.8 & \textbf{1761.4} & \textbf{3.8} & +* \\
 & 2136 & 100000 & 7212.1 & 146.9 & \textbf{11584.5} & \textbf{115.7} & +* & 5792.1 & 113.6 & \textbf{9683.0} & \textbf{112.5} & +* & 2422.5 & 8.3 & \textbf{2708.7} & \textbf{5.1} & +* \\
 & 2136 & 200000 & 8969.4 & 143.5 & \textbf{13673.2} & \textbf{99.8} & +* & 7320.1 & 142.5 & \textbf{11681.5} & \textbf{126.0} & +* & 2541.6 & 6.6 & \textbf{2848.3} & \textbf{5.4} & +* \\
 & 2136 & 500000 & 11565.6 & 109.6 & \textbf{16182.8} & \textbf{56.3} & +* & 9679.1 & 152.2 & \textbf{14395.3} & \textbf{105.8} & +* & 2709.6 & 5.6 & \textbf{3049.6} & \textbf{3.8} & +* \\
 & 2136 & 1000000 & 13669.4 & 82.1 & \textbf{17214.7} & \textbf{37.0} & +* & 11672.2 & 125.3 & \textbf{16342.7} & \textbf{65.2} & +* & 2847.0 & 5.9 & \textbf{3219.9} & \textbf{4.8} & +* \\
\midrule\midrule
\multirow{20}{*}{\rotatebox{90}{ca-GrQc}}
 & 1 & 100000 & \textbf{82.0} & \textbf{0.0} & 75.7 & 7.6 & -* & \textbf{81.9} & \textbf{0.4} & 66.8 & 13.1 & -* & 2.0 & 0.0 & 2.0 & 0.0 & = \\
 & 1 & 200000 & \textbf{82.0} & \textbf{0.0} & 79.4 & 3.9 & -* & \textbf{82.0} & \textbf{0.0} & 70.2 & 10.4 & -* & 2.0 & 0.0 & 2.0 & 0.0 & = \\
 & 1 & 500000 & \textbf{82.0} & \textbf{0.0} & 80.9 & 1.5 & -* & \textbf{82.0} & \textbf{0.0} & 75.1 & 8.0 & -* & 2.0 & 0.0 & 2.0 & 0.0 & = \\
 & 1 & 1000000 & \textbf{82.0} & \textbf{0.0} & 81.7 & 0.7 & -* & \textbf{82.0} & \textbf{0.0} & 77.3 & 6.6 & -* & 2.0 & 0.0 & 2.0 & 0.0 & = \\
 & 12 & 100000 & \textbf{489.4} & \textbf{8.6} & 440.2 & 16.3 & -* & \textbf{537.5} & \textbf{10.3} & 494.9 & 14.8 & -* & 24.0 & 0.0 & 24.0 & 0.0 & = \\
 & 12 & 200000 & \textbf{502.7} & \textbf{5.8} & 462.0 & 10.2 & -* & \textbf{565.9} & \textbf{9.5} & 526.9 & 15.0 & -* & 24.0 & 0.0 & 24.0 & 0.0 & = \\
 & 12 & 500000 & \textbf{509.8} & \textbf{0.6} & 487.7 & 8.0 & -* & \textbf{590.2} & \textbf{5.1} & 563.9 & 9.7 & -* & 24.0 & 0.0 & 24.0 & 0.0 & = \\
 & 12 & 1000000 & \textbf{510.0} & \textbf{0.0} & 500.5 & 6.0 & -* & \textbf{598.8} & \textbf{4.0} & 586.6 & 5.5 & -* & 24.0 & 0.0 & 24.0 & 0.0 & = \\
 & 64 & 100000 & 1319.5 & 16.4 & 1315.2 & 18.1 & = & 1279.8 & 15.9 & \textbf{1365.2} & \textbf{15.1} & +* & 128.0 & 0.0 & 128.0 & 0.0 & = \\
 & 64 & 200000 & \textbf{1407.2} & \textbf{13.2} & 1398.7 & 16.4 & -* & 1387.2 & 17.3 & \textbf{1457.9} & \textbf{16.2} & +* & 128.0 & 0.0 & 128.0 & 0.0 & = \\
 & 64 & 500000 & \textbf{1488.2} & \textbf{8.3} & 1480.6 & 8.2 & -* & 1512.6 & 12.3 & \textbf{1565.1} & \textbf{8.7} & +* & 128.0 & 0.0 & 128.0 & 0.0 & = \\
 & 64 & 1000000 & \textbf{1513.5} & \textbf{6.6} & 1510.0 & 5.6 & -* & 1595.5 & 10.3 & \textbf{1628.4} & \textbf{7.3} & +* & 128.0 & 0.0 & 128.0 & 0.0 & = \\
 & 207 & 100000 & 2149.1 & 21.4 & \textbf{2406.4} & \textbf{14.3} & +* & 2055.0 & 18.9 & \textbf{2381.8} & \textbf{18.8} & +* & 371.8 & 1.9 & \textbf{402.4} & \textbf{1.2} & +* \\
 & 207 & 200000 & 2322.8 & 22.1 & \textbf{2546.5} & \textbf{12.2} & +* & 2230.8 & 13.6 & \textbf{2529.6} & \textbf{14.4} & +* & 394.3 & 1.9 & \textbf{414.0} & \textbf{0.0} & +* \\
 & 207 & 500000 & 2529.5 & 15.5 & \textbf{2691.5} & \textbf{7.7} & +* & 2448.3 & 13.2 & \textbf{2704.3} & \textbf{10.4} & +* & 413.8 & 0.4 & \textbf{414.0} & \textbf{0.0} & +* \\
 & 207 & 1000000 & 2657.6 & 11.0 & \textbf{2749.1} & \textbf{7.8} & +* & 2601.8 & 8.7 & \textbf{2805.3} & \textbf{8.7} & +* & 414.0 & 0.0 & 414.0 & 0.0 & = \\
 & 415 & 100000 & 2707.2 & 23.0 & \textbf{3170.3} & \textbf{12.4} & +* & 2410.7 & 36.3 & \textbf{3049.1} & \textbf{15.7} & +* & 644.0 & 3.1 & \textbf{741.5} & \textbf{2.5} & +* \\
 & 415 & 200000 & 2918.7 & 16.5 & \textbf{3340.8} & \textbf{12.2} & +* & 2767.1 & 24.8 & \textbf{3241.9} & \textbf{12.0} & +* & 685.8 & 2.8 & \textbf{783.5} & \textbf{1.3} & +* \\
 & 415 & 500000 & 3174.5 & 13.5 & \textbf{3506.0} & \textbf{7.3} & +* & 3051.7 & 10.9 & \textbf{3451.1} & \textbf{9.0} & +* & 742.2 & 2.4 & \textbf{824.9} & \textbf{1.2} & +* \\
 & 415 & 1000000 & 3342.0 & 12.8 & \textbf{3575.1} & \textbf{4.8} & +* & 3241.4 & 12.0 & \textbf{3565.9} & \textbf{7.3} & +* & 783.0 & 1.9 & \textbf{830.0} & \textbf{0.0} & +* \\
\midrule\midrule
\multirow{20}{*}{\rotatebox{90}{ca-HepPh}}
 & 1 & 100000 & \textbf{492.0} & \textbf{0.0} & 406.0 & 55.2 & -* & \textbf{467.3} & \textbf{25.1} & 355.8 & 64.1 & -* & 2.0 & 0.0 & 2.0 & 0.0 & = \\
 & 1 & 200000 & \textbf{492.0} & \textbf{0.0} & 429.5 & 40.4 & -* & \textbf{477.4} & \textbf{14.5} & 374.9 & 48.8 & -* & 2.0 & 0.0 & 2.0 & 0.0 & = \\
 & 1 & 500000 & \textbf{492.0} & \textbf{0.0} & 450.4 & 34.4 & -* & \textbf{483.0} & \textbf{0.0} & 394.4 & 49.5 & -* & 2.0 & 0.0 & 2.0 & 0.0 & = \\
 & 1 & 1000000 & \textbf{492.0} & \textbf{0.0} & 470.1 & 27.2 & -* & \textbf{483.0} & \textbf{0.0} & 411.8 & 38.6 & -* & 2.0 & 0.0 & 2.0 & 0.0 & = \\
 & 13 & 100000 & \textbf{1696.2} & \textbf{28.6} & 1448.0 & 49.5 & -* & \textbf{1720.7} & \textbf{38.5} & 1589.8 & 43.5 & -* & 26.0 & 0.0 & 26.0 & 0.0 & = \\
 & 13 & 200000 & \textbf{1776.7} & \textbf{26.7} & 1515.0 & 65.5 & -* & \textbf{1826.8} & \textbf{30.6} & 1657.7 & 35.3 & -* & 26.0 & 0.0 & 26.0 & 0.0 & = \\
 & 13 & 500000 & \textbf{1828.2} & \textbf{11.9} & 1600.4 & 47.2 & -* & \textbf{1958.5} & \textbf{22.7} & 1753.8 & 32.8 & -* & 26.0 & 0.0 & 26.0 & 0.0 & = \\
 & 13 & 1000000 & \textbf{1848.6} & \textbf{13.6} & 1668.0 & 42.9 & -* & \textbf{2040.5} & \textbf{19.4} & 1840.1 & 39.4 & -* & 26.0 & 0.0 & 26.0 & 0.0 & = \\
 & 105 & 100000 & 3744.1 & 44.4 & \textbf{3767.4} & \textbf{40.9} & +* & 3611.6 & 48.7 & \textbf{3891.8} & \textbf{38.0} & +* & \textbf{203.0} & \textbf{1.2} & 196.1 & 1.5 & -* \\
 & 105 & 200000 & \textbf{3992.7} & \textbf{41.1} & 3925.1 & 35.6 & -* & 3871.2 & 38.9 & \textbf{4075.9} & \textbf{30.0} & +* & \textbf{209.9} & \textbf{0.3} & 206.0 & 1.0 & -* \\
 & 105 & 500000 & \textbf{4311.7} & \textbf{24.7} & 4167.2 & 36.4 & -* & 4228.1 & 34.3 & \textbf{4340.7} & \textbf{19.7} & +* & 210.0 & 0.0 & 210.0 & 0.0 & = \\
 & 105 & 1000000 & \textbf{4513.9} & \textbf{23.4} & 4367.9 & 23.5 & -* & 4469.6 & 28.2 & \textbf{4553.5} & \textbf{27.2} & +* & 210.0 & 0.0 & 210.0 & 0.0 & = \\
 & 560 & 100000 & 5902.4 & 62.5 & \textbf{6981.1} & \textbf{39.6} & +* & 5101.8 & 87.8 & \textbf{6774.5} & \textbf{37.5} & +* & 788.2 & 3.3 & \textbf{870.0} & \textbf{4.2} & +* \\
 & 560 & 200000 & 6524.3 & 37.1 & \textbf{7313.5} & \textbf{33.1} & +* & 5962.7 & 69.2 & \textbf{7180.5} & \textbf{34.6} & +* & 841.3 & 3.6 & \textbf{924.2} & \textbf{2.9} & +* \\
 & 560 & 500000 & 7096.3 & 30.4 & \textbf{7769.3} & \textbf{19.5} & +* & 6803.6 & 34.9 & \textbf{7680.2} & \textbf{34.4} & +* & 914.0 & 2.7 & \textbf{999.0} & \textbf{3.4} & +* \\
 & 560 & 1000000 & 7519.9 & 23.4 & \textbf{8106.7} & \textbf{20.8} & +* & 7256.8 & 24.2 & \textbf{8033.7} & \textbf{21.3} & +* & 973.2 & 2.6 & \textbf{1056.0} & \textbf{2.7} & +* \\
 & 1120 & 100000 & 5920.1 & 79.8 & \textbf{8150.2} & \textbf{39.3} & +* & 5118.1 & 75.2 & \textbf{7215.6} & \textbf{66.4} & +* & 1371.4 & 4.3 & \textbf{1578.5} & \textbf{3.4} & +* \\
 & 1120 & 200000 & 6868.1 & 70.5 & \textbf{8788.3} & \textbf{29.7} & +* & 5969.4 & 83.0 & \textbf{8180.5} & \textbf{58.8} & +* & 1454.4 & 4.6 & \textbf{1679.8} & \textbf{2.8} & +* \\
 & 1120 & 500000 & 8129.7 & 49.3 & \textbf{9374.2} & \textbf{20.1} & +* & 7222.4 & 72.4 & \textbf{9048.8} & \textbf{21.2} & +* & 1579.2 & 4.1 & \textbf{1819.3} & \textbf{3.5} & +* \\
 & 1120 & 1000000 & 8796.6 & 31.4 & \textbf{9763.7} & \textbf{15.7} & +* & 8168.0 & 65.6 & \textbf{9491.8} & \textbf{19.2} & +* & 1680.5 & 3.4 & \textbf{1921.7} & \textbf{3.0} & +* \\
\midrule
\bottomrule
\end{tabular}
\end{adjustbox}
\end{table}
\section{Conclusions}
We have shown that our multitasking approach results in small Pareto fronts that allow the solutions to be shared between similar problems. We analyzed the expected time until the introduced multitasking approaches obtain a $(1-1/e)$-approximation for each of the given problems and showed that it's an effective method of solving these problems. Our experimental investigation using the maximum coverage problem demonstrated that the multitasking approaches outperform the classical approaches when given a fixed compute budget for in some scenarios, but that it also was worse in others. This shows that while theoretically the multitasking approach should be more effective than the classical approach for solving these constrained submodular monotone functions, in practice this is highly problem dependent. Notably, even when considering constraints without linear weights we still see a similar pattern between the efficacy of GSEMO vs multitasking GSEMO. Both through theory and experiments, we have shown that the primary strengths of multitasking GSEMO come from problems that are similar enough that no one problem dominates the other. 

There are still many open avenues of investigation. Currently we have focused on submodular monotone problems with uniform constraints. Dynamic constraint problems would be a natural extension, where the constraint can change as the algorithm processes. This can almost be seen as a new static problem were we reuse the prior population as a starting point. However, it would be interesting to be less strict in the requirements and allow approximately submodular functions, non-monotone functions or different classes of constraints such as matroid constraints. Sliding window techniques could also be used to reduce the size of the Pareto front further, which would combine effectively with the improvement the multitasking approach has on the size of the front. This could potentially improve the tradeoff the multitasking approach has to make between the large search space penalizing the smaller problems while giving effectively more generations to the larger problems.

Experimentally, looking at longer generations and more varied constraints would also increase the understanding of the dynamics underlying the multitasking approach. Even a deeper look at varying $k$ or picking different sets of bounds would also give a deeper insight into the problem similarity that drives the effectiveness (or not) of the approach.

\bibliographystyle{unsrt}
\bibliography{bib}

\end{document}